%% file: main.tex
\documentclass{article}
\usepackage{ijcai22}
\input{command.tex}

\usepackage{times,color,enumitem}
\usepackage{soul}
\usepackage{url}
\usepackage[colorlinks,linkcolor=blue]{hyperref}
\usepackage[utf8]{inputenc}
\usepackage[small]{caption}
\usepackage{graphicx,subfigure}
\usepackage{amsmath}
\usepackage{amsthm}
\usepackage{booktabs}
\usepackage{algorithm}
\usepackage{algorithmic}
\usepackage{xspace,multirow}
\usepackage{graphbox}
\newcommand{\method}{{EPPO}\xspace}

\title{Towards Applicable Reinforcement Learning: \\ Improving the Generalization and Sample Efficiency with Policy Ensemble}

\author{
Zhengyu Yang$^1$\footnote{
The work was conducted during Zhengyu Yang's internship at Microsoft Research. The corresponding author is Kan Ren.}\and
Kan Ren$^2$\and
Xufang Luo$^2$\and
Minghuan Liu$^1$\and
Weiqing Liu$^2$\and\\
Jiang Bian$^2$\and
Weinan Zhang$^1$\And
Dongsheng Li$^2$\\
\affiliations
$^1$Shanghai Jiao Tong University\\
$^2$Microsoft Research\\
\emails
\{yzydestiny, minghuanliu, wnzhang\}@sjtu.edu.cn, \\
\{kan.ren, xufluo, weiqing.liu, jiang.bian, dongsli\}@microsoft.com
}

\begin{document}

\maketitle

\input{paper/abs}
\input{paper/intro}
\input{paper/related}
\input{paper/method}

\input{paper/exp}

\input{paper/conclusion}

\bibliographystyle{named}
\bibliography{ijcai22}

\input{paper/supp}

\end{document}

%% file: command.tex
\usepackage{amsmath,amsfonts,bm, amsthm}

\newtheorem{theorem}{Theorem}

\newcommand{\caA}{\ensuremath{\mathcal{A}}} 
\newcommand{\caS}{\ensuremath{\mathcal{S}}} 
\newcommand{\caL}{\ensuremath{\mathcal{L}}} 
\newcommand{\caH}{\ensuremath{\mathcal{H}}} 
\newcommand{\bE}{\ensuremath{\mathbb{E}}} 
\newcommand{\bI}{\ensuremath{\mathbb{I}}} 

\def\eqref#1{equation~\ref{#1}}

\def\1{\bm{1}}

\DeclareMathAlphabet{\mathsfit}{\encodingdefault}{\sfdefault}{m}{sl}
\SetMathAlphabet{\mathsfit}{bold}{\encodingdefault}{\sfdefault}{bx}{n}

\DeclareMathOperator*{\argmax}{arg\,max}

%% file: paper/abs.tex
\begin{abstract}
It is challenging for reinforcement learning (RL) algorithms to succeed in real-world applications like financial trading and logistic system due to the noisy observation and environment shifting between training and evaluation. 
Thus, it requires both high sample efficiency and generalization for resolving real-world tasks.
However, directly applying typical RL algorithms can lead to poor performance in such scenarios.
Considering the great performance of ensemble methods on both accuracy and generalization in supervised learning (SL),
we design a robust and applicable method named Ensemble Proximal Policy Optimization (\method), which learns ensemble policies in an end-to-end manner. 
Notably, \method combines each policy and the policy ensemble organically and optimizes both simultaneously. 
In addition, \method adopts a diversity enhancement regularization over the policy space which helps to generalize to unseen states and promotes exploration.
We theoretically prove \method increases exploration efficacy, 
and through comprehensive experimental evaluations on various tasks, we demonstrate that \method achieves higher efficiency and is robust for real-world applications compared with vanilla policy optimization algorithms and other ensemble methods.
Code and supplemental materials are available at \href{https://seqml.github.io/eppo}{https://seqml.github.io/eppo}.
\end{abstract}

%% file: paper/intro.tex
\section{Introduction}
Compared with simple simulation tasks, it is more difficult for RL algorithms to succeed in real-world applications. 
First, the observation contains much noise and the sampling cost is more expensive in real-world applications.
Second, the environment shifts between the training and evaluation due to the complexity of the real world.
For instance, in financial trading, the noise in the imperfect market information puts forward high requirements on the sample efficiency of the algorithm, and the volatile market requires the algorithms not to overfit to the training environment and retain the ability to generalize to unseen states during evaluation.
However, typical RL algorithms cannot achieve satisfactory performance in these applications. 
Motivated by the superior performance of ensemble methods on improving the accuracy and generalization ability in SL especially for small datasets, 
we resort to ensemble methods to fulfill the aforementioned requirements. In our work, we focus on policy ensemble, which is an integration of a set of \textit{sub-policies}, instead of value function ensemble for some reasons listed below.
i) Value-based methods perform worse than policy-based methods in noisy applications like MOBA games~\cite{ye2020mastering}, card games~\cite{guan2022perfectdou,li2020suphx} and financial trading~\cite{fang2021universal}.
ii) Previous ensemble techniques for RL are mainly applied to SL components, like environment dynamics modeling~\cite{kurutach2018model} and value function approximation~\cite{anschel2017averaged}. 
iii) Policy learning in RL algorithms is critical which is more different from SL. But it is not well studied and thus is worth exploring.

Notice that in many real-world applications such as those mentioned above, Proximal Policy Optimization (PPO)~\cite{schulman2017proximal} is always the first choice of the underlying RL algorithm due to its excellent and stable performances. Therefore, in order to derive an applicable RL algorithm, in this paper, we take PPO as our backbone, and propose a simple yet effective policy ensemble method named Ensemble Proximal Policy Optimization (\method).
\method rigorously treats ensemble policy learning as a first class problem to explicitly address:
i) what is the reasonable yet effective policy ensemble strategy in deep RL and ii) how it helps to improve the performance of policy learning.

Some existing works for policy ensemble aim to attain a diverse set of policies through \textit{individually} training various policies and simply aggregating them ex post factor~\cite{wiering2008ensemble,duell2013ensembles,saphal2020seerl}, which has few guarantees to improve the overall performance due to the neglection of the cooperation among different sub-policies.
The other works incorporate the divide-and-conquer principle to divide the state space, and derive a set of diverse sub-policies accordingly~\cite{ghosh2017divide,goyal2019reinforcement,ren2021probabilistic}. 
But the difficulty in dividing the state space and the unawareness of the sub-policy on the whole state space may significantly hurt the performance especially in \textit{deep} RL.
Furthermore, whether and how the ensemble method would benefit policy optimization still remain unsolved and require additional attention.

In contrast, \method resolves the ensemble policy learning from two aspects.
On one side, we argue that ensemble learning and policy learning should be considered as a whole organic system to promote the cooperation among the sub-policies and guarantee the ensemble performance. 
Thus, \method combines sub-policy training and decision aggregation as a whole and optimizes them under a unified ensemble-aware loss. 
To fully exploit the data and improve sample efficiency, sub-policies are also optimized with the data collected by the ensemble policy which aggregates all the co-training sub-policies for final decision.  
Furthermore, we theoretically prove that the decision aggregation of co-training sub-policies helps in efficient exploration thus improves the sample efficiency.
On the other side, considering that ensemble methods benefit from the diversity among sub-policies and it is difficult to divide the state space to train diverse policies reasonably, 
\method incorporates a diversity enhancement regularization within the policy space to guarantee the diversity and further improve the ensemble performance. 
We empirically found that it can improve policy generalization in real-world applications because the diversity enhancement regularization prevents the sub-policies from collapsing into a singular mode or over-fitting to the training environment, which retains the ability of the ensemble policy to generalize to unseen states.

In a nutshell, the main contributions of the work are threefold:
\begin{itemize}[leftmargin=3mm]
    \item We propose a simple yet effective ensemble strategy in ensemble policy learning and prove that aggregating the co-training sub-policies can promote policy exploration and improve sample efficiency. 
    \item To the best of our knowledge, \method is the first work that adopts the diversity enhancement regularization to attain a diverse set of policies for policy ensemble.
    \item Demonstrated by the experiments on grid-world environments, Atari benchmarks and a real-world application, \method yields better sample efficiency and the diversity enhancement regularization also provides a promising improvement on policy generalization.
\end{itemize}

%% file: paper/related.tex
\section{Background}
\input{paper/preliminary}
\subsection{Related Works}
\paragraph{Ensemble in RL}
The recent works applying ensemble methods in RL are mainly focusing on environment dynamics modeling and value function approximation.
For environment dynamics modeling, several environment models are used to reduce the model variance~\cite{chua2018deep} and stabilize the model-based policy learning~\cite{kurutach2018model}.
As for value function approximation, Q-function ensemble is popular in alleviating the over-estimation~\cite{anschel2017averaged}, encouraging exploration~\cite{lee2021sunrise} and realizing conservative policy learning in offline reinforcement learning~\cite{wu2021uncertainty}.

Nevertheless, the mechanism behind environment dynamics modeling and value function approximation is similar to SL and it has a huge gap with policy learning in RL.
Among the existing works on ensemble policy learning, some works follow the technique used in SL and simply aggregate individually trained policies ex post facto. 
To generate a set of diverse sub-policies, different weight initialization~\cite{fausser2015neural,duell2013ensembles}, training epochs~\cite{saphal2020seerl}, or RL algorithms~\cite{wiering2008ensemble} are used.
Compared with SL tasks, RL agents must take a sequence of decisions instead of making a one-step prediction, which makes the cooperation among sub-policies more important to attain a good ensemble. 
Without considering policy ensemble and policy learning as a whole optimization problem, the cooperation among sub-policies can be neglected, so that these methods have few guarantees to improve the overall performance.
The other works incorporate the divide-and-conquer principle to divide the state space and derive a set of specialized policies accordingly, and then these policies are aggregated to solve the original task, which coincides with the idea of mixture-of-experts~(MOE)~\cite{jacobs1991adaptive}. The essential of the MOE is how to deliver data and obtain a set of  specialized policies (i.e. experts) which focus on different regions of the state space. \method, which proposes a new method (i.e. diversity enhancement regularization) to derive an ensemble of experts and shows better performance in experiments, is also a special case under the paradigm of MOE.
To divide the state space, DnC~\cite{ghosh2017divide} heuristically divides the whole task into several sub-tasks based on the clustering of the initial states while ComEns~\cite{goyal2019reinforcement} and PMOE~\cite{ren2021probabilistic} learn a division principle based on information theory and Gaussian mixture model respectively. 
However, it is hard to divide the state space in many environments and the unreasonable division can damage the performance.
Moreover, the unawareness of the whole state space of the sub-policies caused by the division may significantly hurt the performance especially in \textit{deep} RL scenarios and result in poor ensemble effectiveness.

\paragraph{Diversity Enhancement}
Diversity enhancement, which aims at deriving a set of diverse policies, is mainly used in population-based RL. To enhance the diversity, KL divergence~\cite{hong2018diversity}, maximum mean discrepancy~\cite{masood2019diversity} and determinantal point process~\cite{parker2020effective}
are adopted as bonus during training process.
In policy ensemble, the diversity among sub-policies is also important, but almost all the existing methods apply state space division to enhance the diversity, which may hurt the performance as mentioned before. Thus, we impose the diversity enhancement regularization on the policy space to guarantee the diversity among the sub-policies.

%% file: paper/preliminary.tex
\subsection{Preliminaries}
Sequential decision making process can be formulated as a Markov decision process (MDP), represented by a tuple $\mathcal{M} = \langle \caS, \caA, p, p_0, r \rangle$. 
$\caS = \{s\}$ is the space of the environment states. 
$\caA = \{a\}$ is the action space of the agent. 
$p(s_{t+1}|s_t, a_t): \caS \times \caA \mapsto \Omega(\caS)$ is the dynamics model,
where $\Omega(\caS)$ is the set of distributions over $\caS$.
The initial state $s_0$ of the environment follows the distribution $p_0: S \mapsto \mathbb{R}$.
$r(s,a): \caS \times \caA \mapsto \mathbb{R}$ is the reward function. 
The objective is learning a policy $\pi$ to maximize the cumulative reward $\eta(\pi)=\bE_{\tau \sim \pi} \left[ \sum_{t=0}^{T} r(s_t,a_t)\right]$ where $\tau$ is the trajectory sampled by $\pi$. In this paper, we consider discrete control tasks where $\caA$ is limited and discrete.

%% file: paper/method.tex
\section{Ensemble Proximal Policy Learning}
In this section, we first motivate our design in \method on the training of sub-policies and data collection, and then introduce the details of the learning method.
The overview of the architecture is illustrated in Figure \ref{fig:framework}. 

\begin{figure}
	\centering
	\includegraphics[width=.9\columnwidth]{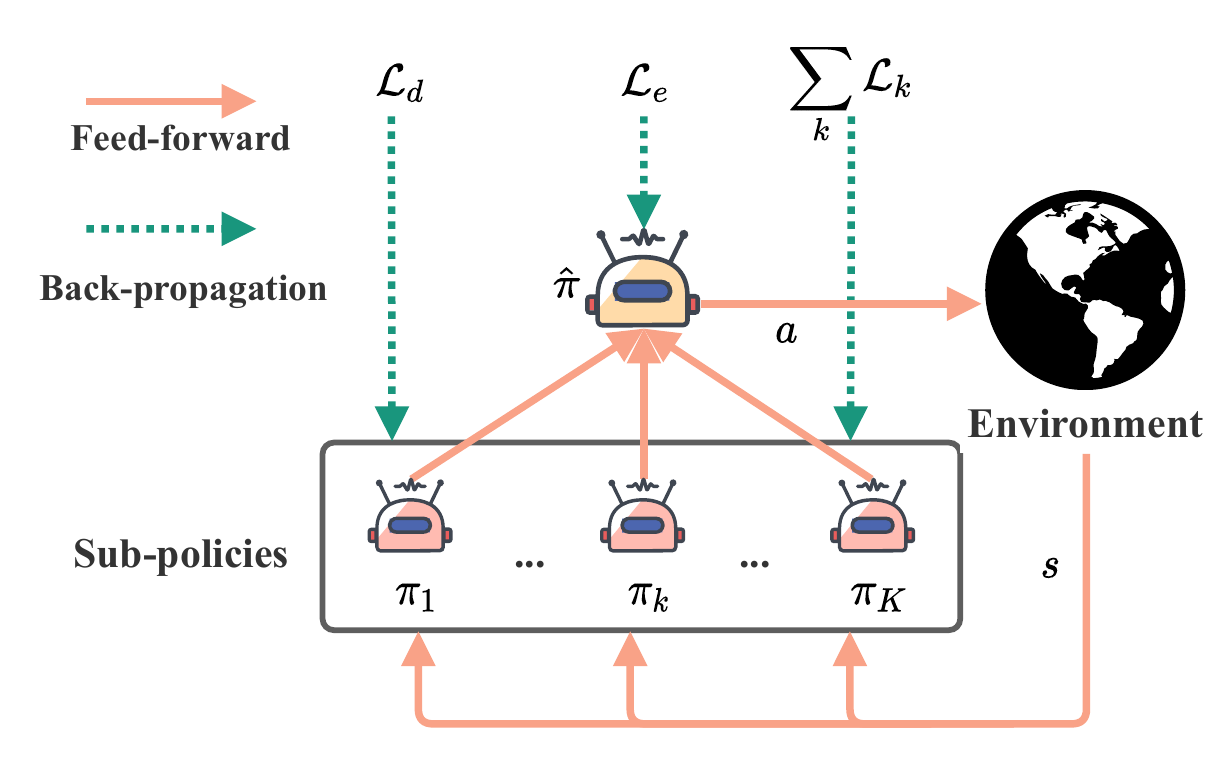}
	\vspace{-10pt}
	\caption{Framework of \method. Only the ensemble policy $\hat{\pi}$ is used to interact with the environment while all the sub-policies are updated simultaneously based on the same collected data.}
	\label{fig:framework}
	\vspace{-5pt}
\end{figure}

\subsection{Policy Ensemble}
As widely used in the literature of ensemble methods~\cite{zhou2002ensembling,anschel2017averaged}, the approximated function is aggregated by a set of base components, each of which can be optimized and work individually in the target task.
Thereafter, we consider maintaining $K$ sub-policies $\{ \pi_{\theta_1}, \pi_{\theta_2},\ldots,\pi_{\theta_K}\}$.
For brevity, we denote $\pi_k$ to represent the sub-policy parameterized by $\theta_k$. Then, 
the ensemble policy $\hat{\pi}$ can be derived through mean aggregation over the sub-policies. 
Formally, for a given state $s$, 
the ensemble policy $\hat{\pi}(\cdot|s)$ is  calculated as the arithmetic mean of the sub-policies:
\begin{equation}
    \label{eq:meanpooling}
    \begin{aligned}
    \hat{\pi}(\cdot|s) = \frac{1}{K}\sum_{k=1}^{K}\pi_k(\cdot|s) ~.
    \end{aligned}
\end{equation}
Note that the parameters of the ensemble policy $\hat{\pi}$ are exactly the whole set of parameters of sub-policies $\{\theta_k\}_{k=1}^K$.
In RL tasks, sample efficiency is a key problem and we would expect that the performance improvement of the designed ensemble method comes from the algorithm itself instead of the more trajectories sampled by more sub-policies. Thus, only the ensemble policy $\hat{\pi}$ is allowed for data collection in \method and the agent samples the action $a \sim \hat{\pi}(\cdot|s)$ from the ensemble policy $\hat{\pi}$ when interacting with the environment. Then the trajectories sampled by the ensemble policy $\hat{\pi}$ will be further used for updating the sub-policies.

\subsection{Policy Optimization}
The previous ensemble works in SL also motivate a fact that better sub-models
lead to better empirical results~\cite{zhang2020diversified}, thus we apply PPO to maximize the expected return (i.e., $\eta(\pi_{k}), 1\leq k \leq K$) of the sub-policies and the loss is defined as:
\begin{equation}
    \label{eq:base_loss}
    \small
    \begin{aligned}
        \caL_k(\pi_{k}) = \bE_{\hat{\pi}'}\! \left[
        \sum_{t=0}^{T}
         \mu \text{KL} \left[ \hat{\pi}'(\cdot|s_t),\pi_{k}(\cdot|s_t)\right] \!- \frac{\pi_{k}(a_t|s_t)}{\hat{\pi}'(a_t|s_t)} \hat{A}_{\hat{\pi}'}(t)\!\right]\!,
    \end{aligned}
\end{equation}
where $\hat{\pi}'$ is the ensemble policy, $\mu$ is an adaptive penalty parameter to constrain the size of the policy update, and $\hat{A}_{\hat{\pi}'}(t)=\hat{A}_{\hat{\pi}'}(s_t,a_t)$ is the advantage function estimated by a generalized advantage estimator (GAE)~\cite{schulman2015high} which describes how much better the action is than others on average.
It is worth noting the data used for optimization is collected by the ensemble policy $\hat{\pi}$, and we only optimize the parameters of sub-policies through policy gradient, 
which does not impose any additional sample cost in the environment as that of a single policy.

Though simply aggregating the predictions of the sub-models in the ensemble has shown effectiveness in improving the performance for supervised tasks \cite{zhou2018diverse}, there is less evidence showing that the aggregation of the sub-policies can improve decision making, because of the large gap between the essential of RL and SL, such as:
i) there may be more than one optimal action at the current state in RL while there is only one ground truth for SL; 
ii) the aggregation of some well-behaved sub-policies may derive undesirable action distribution and lead to bad states, especially when there are many different ways to handle the task and the sub-policies are separately optimized. 
Thus, it is necessary to take the cooperation among sub-policies into consideration and optimize them in a consistent learning paradigm.
In \method, we incorporate an ensemble-aware loss to encourage the cooperation among the sub-policies and ensure a well-behaved ensemble policy $\hat{\pi}$.
The definition of ensemble-aware loss which optimizes the ensemble policy $\hat{\pi}$ by PPO is
\begin{equation}
    \label{eq:ens_loss}
    \small
    \begin{aligned}
    \caL_e(\hat{\pi}) &= \bE_{\hat{\pi}'} \left[
    \sum_{t=0}^{T}
     \mu \text{KL} \left[ \hat{\pi}'(\cdot|s_t),\hat{\pi}(\cdot|s_t)\right] - \frac{\hat{\pi}(a_t|s_t)}{\hat{\pi}'(a_t|s_t)} \hat{A}_{\hat{\pi}'}(t)\right] ~.
    \end{aligned}
\end{equation}
Taking Eq.~(\ref{eq:meanpooling}) into Eq.~(\ref{eq:ens_loss}), we can update the ensemble policy by updating all sub-policies in a unified behavior under the same target.
To some extent, the ensemble-aware loss serves as a regularization that may promote the ensemble performance at the cost of  the performance of sub-policies.

However, there is a potential risk of mode collapse in policy ensemble that all the sub-policies converge to a single policy, which makes policy ensemble useless.
In our method, the problem is even worse because all the sub-policies share a similar training paradigm, which makes these sub-policies tend to behave similarly.
Moreover, \method randomly chooses a sub-policy for action sampling at each step (Eq.(\ref{eq:meanpooling})), so the diversity among sub-policies should promote the exploration of the ensemble policy.
To this end, we propose a diversity enhancement regularization to prevent all sub-policies from collapsing into a singular mode and ensure diverse sub-policies to further improve the ensemble performance. 
Intuitively, in order to enhance the diversity, the regularization should encourage the action distributions proposed by different sub-policies to be orthogonal with each other. 
Specifically, for discrete action space, the diversity enhancement regularization adopted in \method is defined as
\begin{equation}
    \label{eq:div_loss}
    \begin{aligned}
    \caL_{d} = \frac{2}{K(K-1)} \sum_{1\leq i < j \leq K} \sum_a \pi_i(a|s) \pi_j(a|s).
    \end{aligned}
\end{equation}
We note that there are many optional metrics to encourage the diversity among sub-policies such as KL divergence~\cite{hong2018diversity} and MMD~\cite{masood2019diversity} in RL literature, yet we adopt Eq.(\ref{eq:div_loss}) in \method due to its computationally efficiency~\cite{li2012diversity}.

In conclusion, the overall loss to minimize is defined as
\begin{equation}
    \label{eq:tot_loss}
    \begin{aligned}
    \caL = \sum_{k=1}^{K} \caL_k(\pi_{k}) + \alpha \caL_e(\hat{\pi}) +  \beta \caL_{d},
    \end{aligned}
\end{equation}
where $\alpha$ and $\beta$ are the hyper-parameters. 

\subsection{Theoretical Analysis}

\begin{theorem}[Mean aggregation encourages exploration]\label{thm:1}
Suppose $\pi$ and $\{\pi_i\}_{1 \leq i \leq K}$ are sampled from $P(\pi)$, then the entropy of the ensemble policy $\hat{\pi}$ is no less than the entropy of the single policy in expectation, i.e., 
$\bE_{\pi_1,\pi_2,...,\pi_K}\left[ \caH(\hat{\pi}) \right] \ge \bE_{\pi} \left[ \caH(\pi) \right]$.

\end{theorem} 
The proof can be found in Appendix~A in the supplemental materials.
Theorem \ref{thm:1} illustrates that the ensemble policy $\hat{\pi}$ enjoys more effective exploration than the single policy in the policy learning procedure.
Thus, aggregating the sub-policies during training can improve the sample efficiency.
We have also observed the corresponding phenomenon in the experiments, which reflects the effectiveness of the mean aggregation operation for policy ensemble in our method.

%% file: paper/exp.tex
\section{Experiments}
In our paper, we only consider discrete control tasks as it is commonly adopted in real-world scenario applications and the continuous control tasks can be discretized for the ease of optimization.
To evaluate the performance of \method, we conduct experiments on Minigrid~\cite{gym_minigrid}, Atari games~\cite{bellemare2013arcade} and financial trading \cite{fang2021universal}, which span the simulated tasks and real-world applications. 

The experiments and the analysis in this section are led by the following two research questions (RQs). 
{\bf RQ1}: Does our method achieve higher sample efficiency through policy ensemble? 
{\bf RQ2}: Is the generalization performance of our method better than the other compared methods?

\subsection{Compared Methods} We compare \method with the following baselines including two variants of EPPO.
\begin{itemize}[leftmargin=3mm]
    \item {\bf PPO}~\cite{schulman2017proximal} is a state-of-the-art policy optimization method which has been widely used in real-world applications~\cite{fang2021universal,ye2020mastering}. 
    \item {\bf PE} (Policy Ensemble) is based on a traditional policy ensemble method~\cite{duell2013ensembles} which trains $K$ policies \textit{individually} by PPO and then aggregates them. In our paper, we consider two aggregation operation, i.e., majority voting and mean aggregation, which are defined as $\hat{\pi}(a|s)=\frac{1}{K}{\sum_{k=1}^K \bI((\mathop{\arg\max}_{a'}\pi_k(a'|s))==a)}$ and Eq.~(\ref{eq:meanpooling}), respectively. 
    We denote these two methods as {\bf PEMV} and {\bf PEMA} for short.
    \item {\bf DnC}~\cite{ghosh2017divide} partitions the initial state space into $K$ slices, and optimizes a set of policies each on different slices. During training, these policies are periodically distilled into a center policy that is used for evaluation.
    \item {\bf ComEns}~\cite{goyal2019reinforcement} uses an information-theoretic mechanism to decompose the policy into an ensemble of primitives and each primitive can decide for themselves whether they should act in current state. 
    \item {\bf PMOE}~\cite{ren2021probabilistic} applies a routing function to aggregate the sub-policies and deliver the data to different sub-policies during optimization. 
    \item {\bf SEERL}~\cite{saphal2020seerl} uses a learning rate schedule to get multiple policies in a round and selects a set of policies for ensemble according to performance and diversity.
    \item {\bf \method} is our proposed method described above, which has two other variants for ablation study: {\bf \method-Div} is the method \textit{without} diversity enhancement regularization defined in Eq.~(\ref{eq:div_loss}) and {\bf \method-Ens} is the method \textit{without} ensemble-aware loss defined in Eq.~(\ref{eq:ens_loss}). 
\end{itemize}

Since we focus on policy ensemble, we omit Q-function ensemble methods like Sunrise~\cite{lee2021sunrise}.
In all compared baselines, we take PPO as the base policy optimization method for fairness; 
besides, they have roughly the same number of parameters (i.e., $K$ times the parameter size of PPO and we set $K=4$ as default for all experiments) and the same number of samples collected in one training epoch.

\subsection{Improved Efficiency on Minigrid}
We first investigate whether EPPO can improve the sample efficiency. In this part, we consider two partial-observable environments with sparse reward in Minigrid~\cite{gym_minigrid}: \textit{Distributional-Shift} and \textit{Multi-Room}, as shown in Figure~\ref{fig:minigrid}, where the agent aims at reaching the given target position and a nonzero reward is provided only if the target position is reached.
Specifically, the place of the second line of the lava in \textit{Distributional-Shift} is reset, and the shape of \textit{Multi-Room} is regenerated during the reset procedure. 
Due to the more complex structure and longer distance between the start position and the goal, \textit{Multi-Room} is more difficult.

\begin{figure*}[htb]
\small
	\centering
	\subfigure{
	    \label{fig:env1}
		\begin{minipage}[t]{0.24\textwidth}
			\centering
			\includegraphics[align=c,width=0.75\textwidth]{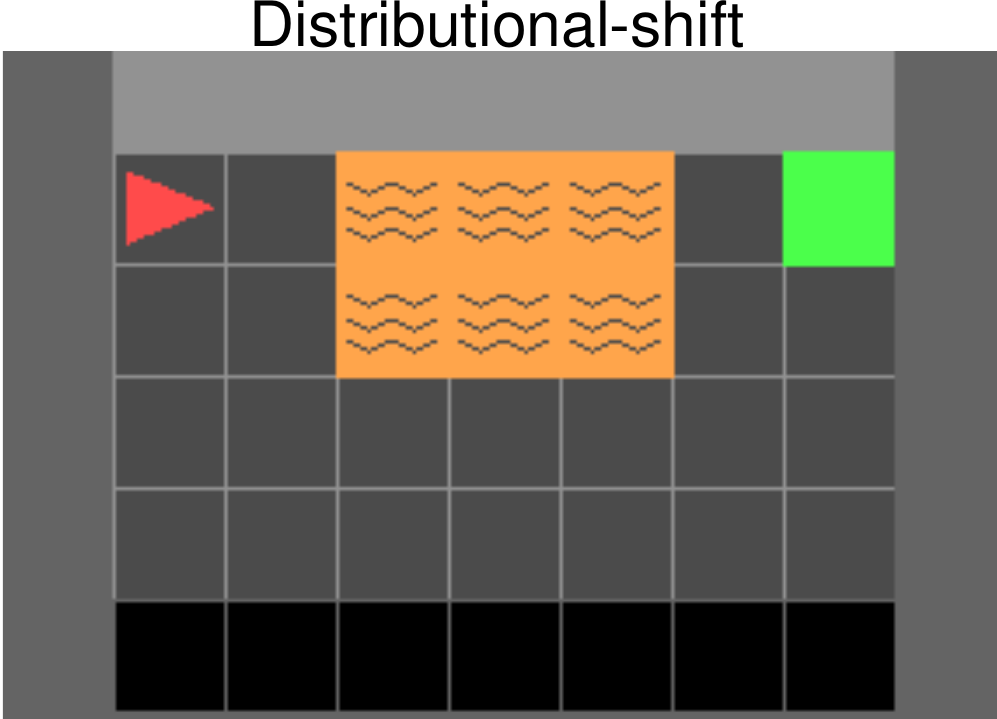}
		\end{minipage}
	}
	\subfigure{
	    \label{fig:lava_res}
		\begin{minipage}[t]{0.72\textwidth}
			\centering
			\includegraphics[align=c,width=0.99\textwidth]{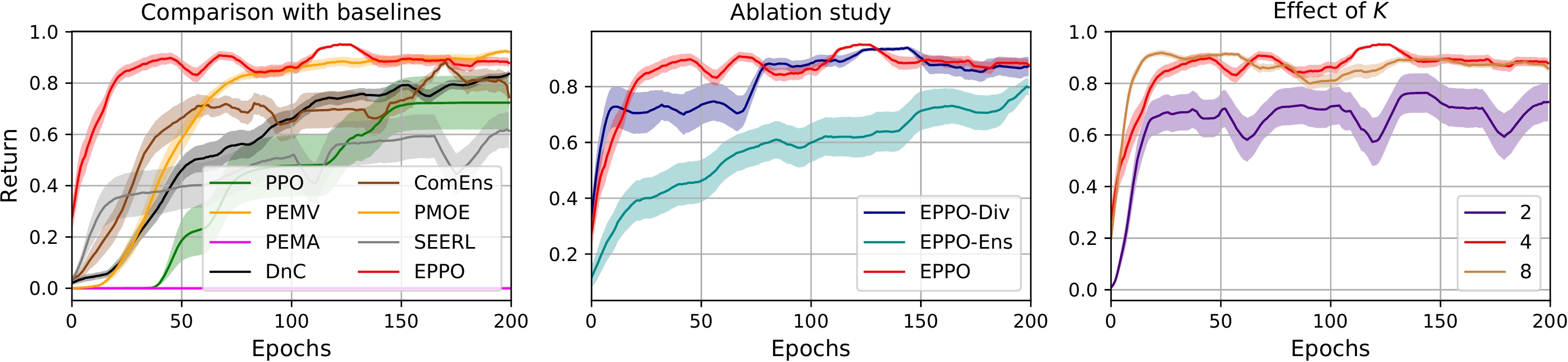}
		\end{minipage}
	}
	
	\subfigure{
	    \label{fig:env2}
		\begin{minipage}[t]{0.24\textwidth}
			\centering
			\includegraphics[align=c,width=0.72\textwidth]{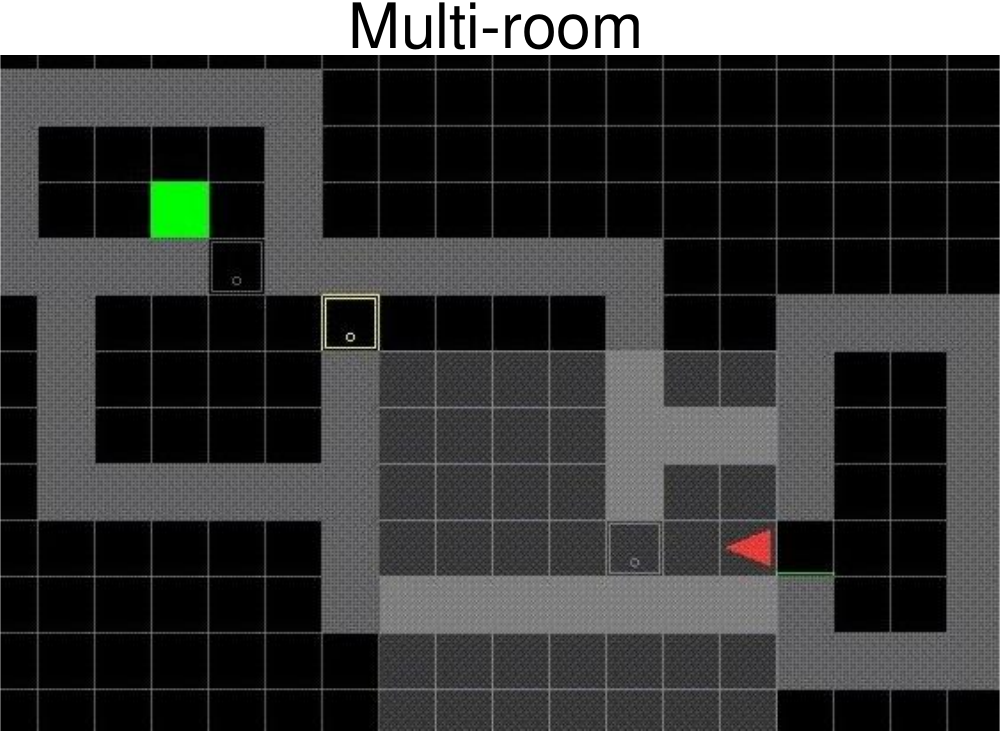}
		\end{minipage}	
	}
	\subfigure{
	    \label{fig:mroom_res}
		\begin{minipage}[t]{0.72\textwidth}
			\centering
			\includegraphics[align=c,width=0.99\textwidth]{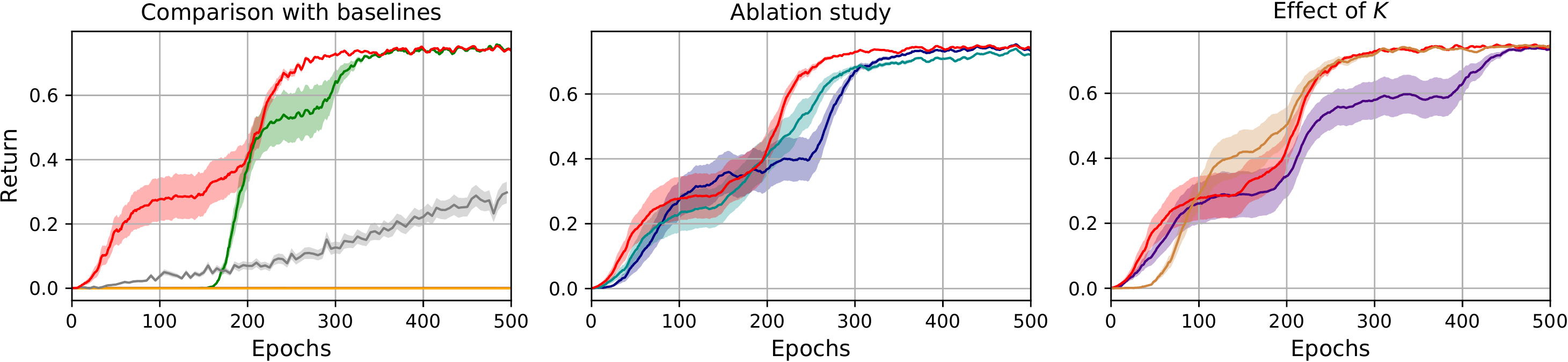}
		\end{minipage}
	}
	
	\caption{The first column gives a snapshot of the environments where the red triangle and green square represent the position and the goal of the agent respectively.
	Learning results on Minigrid conducted with 5 random seeds. The top and bottom rows show the information about \textit{Distributional-Shift} and \textit{\textit{Multi-Room}}, respectively. 
	{\bf First column}: A snapshot of the environment where the red triangle and green square represent the position and the goal of the agent respectively. 
	{\bf Second column}: Learning curves of all the compared methods.
	{\bf Third column}: Learning curves of \method and its variants. 
	{\bf Last column}: Learning curves of \method when $K$ is set to different values.
	} \label{fig:minigrid}
\end{figure*}
As shown in Figure~\ref{fig:minigrid}, \method enjoys the best sample efficiency in both environments~({\bf RQ1}).
We notice that PEMA and PEMV fail (i.e., $\text{return}=0$) in both environments while PPO can get better performance.
And we conclude the failure from two perspectives. 
First, considering the number of samples needed for PPO to get a positive reward is large, PEMA and PEMV may need $K$ times samples for a positive reward because they individually train $K$ PPO policies. 
Second, due to the asynchrony among the sub-policies of PE method, the useful knowledge can be overwhelmed thus neglected due to the aggregation operation by the worthless knowledge. 
Their failure implies the necessity of an ensemble-aware loss and explains the rationality of sampling with ensemble policy $\hat{\pi}$.
For SEERL, its final performance is even worse than that at the last epoch because the ex-post selection cannot ensure an improved performance, which further emphasizes the necessity of the joint optimization of the sub-policies.
Moreover, the failure of DnC, ComEns and PMOE in \textit{Multi-Room} environment can be attributed to the division operation on the state space that may not only unreasonably divide the space but also hinder the ability on exploration to the full environment. 
And comparing the performance of the methods based on state space division and \method, we find that \method consistently achieves better performance, which implies that a diversity enhancement regularization is a better choice for diversity enhancement in policy ensemble.

The ablation study in Figure~\ref{fig:minigrid} shows that \method outperforms its two variants, thus both of diversity enhancement regularization and ensemble-aware loss appear to be crucial for the superior performance of \method.
In addition, the returns of both \method and its variants improve quickly at first, which confirms the result in Theorem~\ref{thm:1} that mean aggregation encourages exploration.
In the last column of Figure~\ref{fig:minigrid}, we analyze the effect of using various number of sub-policies in \method. 
The results indicate that an extremely small $K$ cannot lead to a good performance and the performance cannot be further improved by increasing $K$ by a large margin.
When $K$ is a small value like $2$, the sub-policies tend to have fewer overlaps in the action space, thus the mean aggregation operation  is unable to extract the valuable information from sub-policies and leads to worse performance.

\subsection{Comparative Evaluations on Atari Games}
Having seen the superior performance of \method in environments with sparse reward, we also want to evaluate \method on more difficult and widely used benchmarks. Particularly, we follow the settings in~\cite{saphal2020seerl} and choose four environments from Atari games as the testbed.
As shown in Table~\ref{ta:atari}, \method can still consistently achieve the best performance in 10M environment steps, suggesting a better sample efficiency~({\bf RQ1}). 

\begin{table}[hbt]
    \centering
    \scalebox{0.9}{
    \begin{tabular}{c|cccc}
    \hline
        ~ & Alien & Amidar & Pong & Seaquest  \\ \hline
        PPO & 1174.6 & 283.8 & 20.8 & 1110.2  \\ 
        PEMV & 678.0 & 74.3 & 6.9 & 364.2  \\ 
        PEMA & 815.2 & 113.8 & 7.6 & 563.4  \\ 
        DnC & 158.0 & 41.5 & -21.0 & 185.0  \\ 
        ComEns & 351.6 & 51.0 & -20.7 & 504.2  \\ 
        PMOE& 1488.2 & 247.3 & 3.0 & 1800.5  \\ 
        SEERL  &  1127.8 &  155.0 &  20.0 &  928.4 \\
        \method-Div &  1173.2 &  304.6 &  19.4 &  1580.8 \\
        \method-Ens &  1651.2 &  311.8 &  20.8 &  1816.6 \\
        \method & \bf 1984.0 & \bf 439.7 & \bf 20.9 & \bf 1881.2 \\ \hline
    \end{tabular}
    
    }
    \caption{Performance on Atari games at 10M interactions. All results represent the average over 100 episodes of 5 random training runs.
    Bold font represents the best results.
    }\label{ta:atari}
\end{table}

\subsection{Generalizable Application: A Financial Trading Instance}

\begin{table}[hbt]
\centering

\scalebox{0.9}{
\begin{tabular}{c|ll}
\hline
\multicolumn{1}{c|}{\multirow{2}{*}{Phase}} & \multicolumn{2}{c}{Dataset 1801-1908} \\ 
\multicolumn{1}{c|}{}                       & \# order & \multicolumn{1}{c}{Time Period} \\ 
\hline
Training                                   & 845,006 & 01/01/2018 - 31/12/2018 \\
Validation                                 & 132,098 & 01/01/2019 - 28/02/2019 \\
Test                                       & 455,332 & 01/03/2019 - 31/08/2019 \\ 
\hline\hline
\multicolumn{1}{c|}{\multirow{2}{*}{Phase}} &  \multicolumn{2}{c}{Dataset 1807-2002}     \\ 
\multicolumn{1}{c|}{}                       & \# order & \multicolumn{1}{c}{Time Period} \\ 
\hline
Training                                   &  854,936 & 01/07/2018 - 30/06/2019 \\
Validation                                 &  163,140 & 01/07/2019 - 31/08/2019 \\
Test                                       &  428,846 & 01/09/2019 - 29/02/2020 \\ 
\hline
\end{tabular}
}
\caption{The dataset statistics of financial order execution task.}\label{tab:dataset}
\end{table}
To evaluate the generalization ability, we conduct experiments on order execution \cite{fang2021universal} which is a fundamental yet challenging problem in financial trading.
In order execution, the environments are built upon the historical transaction data, and the agent aims at fulfilling a trading order which specifies the date, stock id and the amount of stock needed to be bought or sold. 
In particular, the environments are usually formulated as training, validation and test phases each of which is corresponding to a specific time range. 
To be specific, \textit{training environment} and \textit{validation environment} are used for policy optimization and policy selection respectively.
\textit{Test environment} is unavailable during training.  
Due to the shift of macroeconomic regulation or other factors in different time, \textit{test environment} may differ a lot to the \textit{training} and \textit{validation environment}.
Thus, the selected policy has to make decisions in unfamiliar states during testing and the performance in \textit{test environment} is a good surrogate evaluation of the generalization ability.

Following \cite{fang2021universal}, the reward is composed of price advantage (PA) and market impact penalty where PA encourages the policy to get better profit than a baseline strategy.
Specifically, we take TWAP as the baseline strategy, which equally splits the order into $T$ pieces and evenly executes the same amount of shares at each timestep during the whole time horizon.
And an increase of 1.0 in PA can bring about a 0.5\% annual return with 20\% daily turnover rate. 
The averaged PA and reward of the orders in the \textit{test environment} are taken as evaluation metrics in order execution.
We conduct experiments on the two large datasets 1801-1908 and 1807-2002 published in 
\cite{fang2021universal} and the statistics of the datasets can be found in Table~\ref{tab:dataset}. 

The results of different methods are reported in Table~\ref{tab:pa}. 
As expected, \method achieves the best performance in both PA and reward in two datasets, suggesting that our proposed method has great potential in generalizing to unseen states~({\bf RQ2}). 
And we find that PEMV has a worse reward than PPO in 1801-1908, which implies individually training sub-policies has no guarantee on the ensemble performance, thus an ensemble-aware loss, i.e., $L_e$ in Eq.~(\ref{eq:ens_loss}) which can encourage the coordination among sub-policies, is necessary. 
In addition, the performance degradation of \method-Ens also illustrates the importance of the ensemble-aware loss. 
Moreover, from the comparison between \method and \method-Div, we find that the diversity enhancement regularization further improves the generalization performance. This phenomenon coincides with the observations in SL \cite{zhou2002ensembling} that the diversity among sub-models can reduce the variance, alleviate the over-fitting problem and improve the generalization performance of the ensemble method.

\begin{table}[!ht]
    \centering
    \scalebox{0.9}{
    \begin{tabular}{c|cccc}
    \hline
        Dataset & \multicolumn{2}{c}{1801-1908} & \multicolumn{2}{c}{1807-2002}   \\ \hline
        Metric & PA & Reward & PA & Reward  \\ \hline
        PPO & 7.43 & 4.57 & 5.30 & 2.75  \\ 
        PEMV & 7.47 & 4.41 & 6.03 & 3.44  \\ 
        PEMA & 7.87 & 5.00 & 5.98 & 3.42  \\ 
        DnC & 7.99 & 5.47 & 5.36 & 2.75  \\ 
        ComEns & 7.70 & 4.79 & 4.59 & 1.32  \\ 
        PMOE & 3.12 & -0.03 & 3.09 & 1.43  \\ 
        SEERL & 7.03 & 5.04 & 5.52 & 3.51  \\ 
        \method-Div & 8.38 & 5.51 & 6.21 & 3.30  \\ 
        \method-Ens & 6.38 & 3.87 & 5.51 & 3.51  \\ 
        \method & \bf 8.82 & \bf 5.99 & \bf 6.31 & \bf 3.57 \\ \hline
    \end{tabular}
    }
    \caption{Test performance on order execution task; the higher metric value means the better performance. The results are the average of all test orders over ten random training runs.
    } \label{tab:pa}
    
\end{table}

To evaluate the effect of $K$, we conduct experiments when $K \in \{2,4,8\}$ in dataset 1801-1908 and the results are shown in Table~\ref{tab:diffk}.
Similar to the experimental results in Minigrid, $K=2$ leads to a poor performance, which can still be attributed to the difficulty of getting consensus during aggregation when $K$ is small. 
In addition, a larger $K$ does not always lead to better performance. 

\begin{table}
\centering
\begin{small}
\begin{tabular}{c|ccc} 
\hline
$K$ & 2  & 4 & 8 \\ \hline
PA     & 7.49 & \bf 8.82 & 8.64  \\ 
Reward & 4.42 & \bf 5.99 & 5.83  \\ \hline
\end{tabular}
\caption{Results on different $K$. } \label{tab:diffk} 
\end{small}
\end{table}

\subsection{Analysing the Policy Diversity}
After showing the performance of diversity enhancement regularization in improving the sample efficiency and policy generalization, we further verify that the regularization does diversify the sub-policies. 
Motivated by ~\cite{hong2018diversity}, we utilize the action disagreement (AD) to measure the diversity among sub-policies, which is defined as
\begin{equation}
        \small
        \frac{\sum_{i,j} \sum_{s \in M} \mathbb{I}(\mathop{\argmax}_{a} \pi_i(a|s) \neq \mathop{\argmax}_{a} \pi_j(a|s))}{|M|K(K-1)},
\end{equation}
where $M$ is a set of states.
In dataset 1801-1908, the AD values of \method and \method-Div are 15.9\% and 14.3\%, respectively, which demonstrates the ability of diversity enhancement regularization in improving diversity.

%% file: paper/conclusion.tex
\section{Conclusion and Future Work}
In this paper, we focus on ensemble policy learning and propose an end-to-end ensemble policy optimization framework called EPPO that combines sub-policy training and policy ensemble as a whole. In particular, EPPO updates all the sub-policies simultaneously under the ensemble-aware loss with a diversity enhancement regularization.
We also provide a theoretical analysis of EPPO on improving the entropy of the policy which leads to better exploration. 
Extensive experiments on various tasks demonstrate that \method substantially outperforms the baselines for both sample efficiency and policy generalization performance.
In the future, we plan to incorporate more flexible sub-policy ensemble mechanisms and dive deeper into the mechanism behind the ensemble policy learning.

%% file: paper/supp.tex
\newpage
\onecolumn
\appendix
\section{Proof of Theorem 1} \label{app:proof}
\setcounter{theorem}{0}
\begin{theorem}[Mean aggregation encourages exploration] \label{thm:a1}
Suppose $\pi$ and $\{\pi_i\}_{1 \leq i \leq K}$ are sampled from $P(\pi)$, then the entropy of the ensemble policy $\hat{\pi}$ is no less than the entropy of the single policy in expectation, i.e., 
$\bE_{\pi_1,\pi_2,...,\pi_K}\left[ \caH(\hat{\pi}) \right] \ge \bE_{\pi} \left[ \caH(\pi) \right]$.
\end{theorem} 

\begin{proof}
\begin{equation}
\begin{aligned}
    H(\hat\pi)-\frac{1}{K}\sum_{k=1}^{K} \caH(\pi_{k}) &= -\sum_a \frac{\sum_{k=1}^{K}\pi_{{k}}(a|s)}{K} \log \left(\frac{\sum_{k'=1}^{K}\pi_{{k'}}(a|s)}{K}\right) \\
    &~~~~~ +\frac{1}{K} \sum_{k=1}^{K} \sum_a \pi_{k}(a|s) \log(\pi_{k}(a|s))\\
    & = -\frac{1}{K} \sum_{k=1}^{K} \sum_a \pi_{k}(a|s) \left(\log \left(\frac{\sum_{k'=1}^{K}\pi_{{k'}}(a|s)}{K}\right)-\log(\pi_{k}(a|s))\right) \\
    & = -\frac{1}{K} \sum_{k=1}^{K} \sum_a \pi_{k}(a|s) \log \left(\frac{\sum_{k'=1}^{K}\pi_{{k'}}(a|s)}{K\pi_{{k}}(a|s)}\right)\\
    & \geq -\frac{1}{K} \sum_{k=1}^{K}\log\left(\sum_a \pi_{k}(a|s) \frac{\sum_{k'=1}^{K}\pi_{{k'}}(a|s)}{K\pi_{k}(a|s)}\right) ~~~~~~~~~\text{(Jensen's inequality)}\\ 
    & = -\frac{1}{K} \sum_{k=1}^{K}\log \left(\frac{1}{K} \sum_a \sum_{k'=1}^{K}\pi_{{k'}}(a|s)\right)\\
    & = -\frac{1}{K} \sum_{k=1}^{K} \log(1) = 0 ~.\nonumber \\
    \Rightarrow H(\hat\pi) &\ge \frac{1}{K}\sum_{k=1}^{K} \caH(\pi_{k})
\end{aligned}
\end{equation}

Then
\begin{equation}
\begin{aligned}
    \bE_{\pi_1,\pi_2,...,\pi_K}[\caH(\hat{\pi})] &\ge \bE_{\pi_1,\pi_2,...,\pi_K}\left[\frac{1}{K} \sum_{k=1}^{K}\caH(\pi_k)\right]\\
    &\ge \frac{1}{K} \sum_{k=1}^{K} \bE_{\pi_1,\pi_2,...,\pi_K}[\caH(\pi_k)] \\
    &\ge \frac{1}{K} \sum_{k=1}^{K} \bE_{\pi} [\caH(\pi)] \\
    &\ge \bE_\pi \left[ \caH(\pi) \right] ~.\nonumber
\end{aligned}
\end{equation}
\end{proof}